\documentclass[11pt]{article}

\usepackage[margin=0.9in]{geometry}
\usepackage[T1]{fontenc}
\usepackage{lmodern}
\usepackage{microtype}
\usepackage{amsmath,amssymb,amsthm,mathtools}
\usepackage{etoolbox}
\usepackage[colorlinks=true,linkcolor=black,urlcolor=blue,citecolor=red]{hyperref}

\AtBeginEnvironment{thebibliography}{%
  \setlength{\itemsep}{0pt}%
  \setlength{\parsep}{0pt}%
  \setlength{\parskip}{0pt}%
}
\pretocmd{\bibitem}{\vspace{-1pt}}{}{}

\allowdisplaybreaks[2]

\theoremstyle{plain}
\newtheorem{theorem}{Theorem}
\newtheorem{lemma}{Lemma}
\newtheorem{proposition}{Proposition}
\newtheorem{corollary}{Corollary}
\theoremstyle{remark}

\DeclareMathOperator*{\argmax}{arg\,max}
\newcommand{\E}{\mathbb E}
\newcommand{\Pp}{\mathbb P}
\newcommand{\1}{\mathbf 1}
\newcommand{\Upd}{\mathsf U}

\title{Thompson Sampling Is $2$-Competitive for Mistakes}
\author{Mark Sellke \and Gregory Valiant}
\date{}

\begin{document}
\maketitle

\begin{abstract}
\noindent
We consider Bayesian bandit models and prove that Thompson sampling makes at most twice the expected number of \emph{mistakes} (selections of a suboptimal arm) as any other policy.
Our analysis applies as long as the latent arm processes are independent and each arm evolves only when played.
For stochastic bandits with best arm defined via mean reward, this confirms a conjecture of \cite{guha-munagala-2014}, where the factor $2$ is already best possible. 
The result holds under any nonincreasing sequence of round weights, including fixed horizon and geometric discounting.
\end{abstract}

\section{Introduction}

In the classical stochastic bandit model, a learner repeatedly
chooses from a fixed collection of arms.  Each play of arm \(i\) produces an
independent draw from an unknown reward distribution with mean \(\mu_i\), and
the learner observes only the reward from the chosen arm.  
The usual objective is to maximize expected cumulative reward, typically under the guise of minimizing the regret compared to the best fixed action
\cite{lai-robbins-1985,auer-cesabianchi-fischer-2002,bubeck-cesabianchi-2012,lattimore-szepesvari-2020}.
Our objective in this paper is instead to minimize the number of \emph{mistake} rounds in which the designated best arm is not chosen.
We follow a Bayesian model in which the unknown arms are
drawn independently from specified priors.  
Given such a prior, the well-studied Thompson sampling algorithm draws, at each time $t$, a fresh posterior sample and plays the best arm prescribed by it.  
\cite{guha-munagala-2014} conjectured that the expected mistake count of Thompson sampling is at most twice the (Bayesian optimal) mistake-minimizing policy, for any fixed time horizon. 
They proved this claim in the case of two arms, where the factor of $2$ is already unimprovable.

In this paper we prove this conjecture, and show that it in fact holds in much greater generality.
Namely each arm has a random real-valued score and a feedback sequence, whose
$k$-th element is revealed on the arm's $k$-th play, and the arms are independent
under the prior.
In the special case that each arm has a latent reward distribution, its
feedback sequence consists of conditionally IID draws from that distribution,
and its score is the distribution's mean, we recover the stochastic bandit
setting of \cite{guha-munagala-2014}.
By averaging over the time horizon, the result easily extends to any decreasing discount rate (the most notable being geometric).

Our argument proceeds by establishing a $1$-step estimate which allows us to compare the Bellman recursions for the optimal policy and for Thompson sampling. 
The estimate is proved by analyzing a counterfactual simulation scheme; the crucial step is a careful swapping of unobserved arms between two independent copies of the world.

\paragraph{Related work.}
\cite{thompson-1933} introduced posterior sampling in 1933, and
\cite{chapelle-li-2011} later demonstrated its strong empirical performance.
Theoretical work since then has mainly studied regret.  The literature includes
finite-time analyses for stochastic bandits
\cite{agrawal-goyal-2017,kaufmann-korda-munos-2012}, extensions to structured
and incentive-constrained problems
\cite{russo-vanroy-2013-posterior,agrawal-goyal-2013-linear,hamidi-bayati-2020,sellke-slivkins-2020},
and the information-ratio method, which also connects posterior sampling to
mirror descent
\cite{russo-vanroy-2014,russo-vanroy-2014-ids,lattimore-gyorgy-2020}.
These regret analyses compare the learner with an oracle that knows the
environment and always selects an optimal arm.  A separate literature studies
Bayes-optimal adaptive policies.  For geometrically discounted Bayesian bandits
with independent arms that evolve only when played, the Gittins theorem
identifies the policy that chooses an arm of largest Gittins index as optimal
\cite{gittins-1979,gittins-glazebrook-weber-2011,scully-terenin-2025}.
For finite horizons, \cite{guha-munagala-2013-approx} gives polynomial-time
constant-factor approximations to Bayes-optimal reward-maximizing policies,
including side-constrained and explore-then-exploit variants.

Evaluating learning algorithms by counting their mistakes is a well-studied setting in online classification, where
the learner predicts labels for a sequence of presented examples \cite{littlestone-1988}.  Later work studies bandit feedback
\cite{daniely-helbertal-2013}, optimal randomized prediction
\cite{filmus-hanneke-mehalel-moran-2023}, and self-directed learners that choose
the examples adaptively \cite{devulapalli-hanneke-2024}.  These settings reveal
the correct label or at least whether each prediction was correct. For us, a play
reveals only feedback from the selected arm, not whether that arm was best; our
guarantee compares Bayesian expected mistakes multiplicatively with the
Bayes-optimal adaptive policy.

\paragraph{Acknowledgement}

M.S. thanks Janardhan Kulkarni and Kamesh Munagala for introducing him to this problem.
The proof in this paper is due to an early version of GPT-5.6.

\section{Model and results}

Our model assumes independence across arms but permits arbitrary dependence
between the score and observations of a single arm.  A fixed tie-breaking rule
designates one of the arms with largest score as best.  The observation streams
determine what the learner can discover, and each play reveals the next
observation in the chosen arm's stream.

\subsection{Independent arms}

\paragraph{Scores and observation sequences.}

We take \(\mathcal I\subseteq\mathbb N\) to be a nonempty finite or countably
infinite set of arms.  For each arm \(i\), let \(\mathsf O_i\) be a standard
Borel observation space.  We collect its hidden score and observation sequence in
\[
        Z_i=(S_i,O_{i,1},O_{i,2},\ldots)
        \in \mathbb R\times\mathsf O_i^{\mathbb N}.
\]
Here \(S_i\) is the arm's fixed unknown score, and \(O_{i,k}\) is the
observation revealed when the arm is played for the \(k\)-th time.
We assume the variables \((Z_i)_{i\in\mathcal I}\) are independent under the prior
\(\Pi_0\).  
Within one arm, this play-indexed observation sequence need not be
independent, identically distributed, or stationary, and it may depend
arbitrarily on \(S_i\).

We assume that the largest score is attained almost surely:
\begin{equation}
        \Pp\!\Big(
        \exists i\in\mathcal I:\;
        S_i=\sup_{j\in\mathcal I}S_j
        \Big)=1.
        \label{eq:max-attained}
\end{equation}
When multiple arms attain the largest score, we break ties by choosing the
smallest index.  Thus the best arm is defined by
\begin{equation}
        I^\star=
        \min\big\{i\in\mathcal I:
        S_i=\sup_{j\in\mathcal I}S_j
        \big\}.
        \label{eq:best-arm}
\end{equation}
The same tie-breaking convention is used for posterior samples.
Thus, when scores tie, playing a maximizing arm other than \(I^\star\) still
counts as a mistake.

\paragraph{Interaction and posterior updating.}

In round \(t\), the learner plays an arm \(A_t\).  We set \(N_i(0)=0\), and for
\(t\ge1\) let
\[
        N_i(t)=\sum_{s=1}^t\1_{\{A_s=i\}}
\]
be the number of times arm \(i\) has been played through round \(t\).  If
\(A_t=i\), the learner observes \(O_{i,N_i(t-1)+1}\) and incurs loss
\(\1_{\{A_t\ne I^\star\}}\).  The scores and observation sequences are sampled
before the interaction begins and are initially hidden from the learner.  The value
\(O_{i,k}\) is revealed only on the \(k\)-th
play of arm \(i\), regardless of the calendar round.  Thus the policy chooses
which hidden observation is revealed next but cannot change the unrevealed
values or their dependence.

A policy is a measurable rule that chooses each action from the past actions
and observations, possibly using private randomness independent of
\((Z_i)_{i\in\mathcal I}\).  Unless
stated otherwise, expectations average over both sources of randomness.

After round \(t\), the unrevealed part of arm \(i\) is
\[
        Z_{i,t}^{\mathrm{rem}}
        =
        (S_i,O_{i,N_i(t)+1},O_{i,N_i(t)+2},\ldots).
\]
We denote by \(\Pi_t\) a regular conditional distribution of
\((Z_{i,t}^{\mathrm{rem}})_{i\in\mathcal I}\) given the realized history.  When
analyzing play after this history, we reindex each residual stream from one and
write \(O_{i,k}\) for \(O_{i,N_i(t)+k}\), so \(O_{i,1}\) is arm \(i\)'s next unread
observation.  Because an observation from one arm reveals
nothing about the others, these residual variables remain independent under
\(\Pi_t\), and only the chosen arm's posterior changes.

We call a posterior \(\Pi\) for the residual variables \emph{admissible} if
they are independent under \(\Pi\) and a best arm exists almost surely.  The initial prior \(\Pi_0\) is admissible by assumption, which implies the same for each future
\(\Pi_t\) by induction.

\paragraph{Thompson sampling and loss.}

At the start of round \(t+1\), Thompson sampling draws a score vector
\(\widetilde S\) from the posterior distribution of
\(S=(S_i)_{i\in\mathcal I}\) and plays
\[
        \widetilde I^\star
        =
        \min\argmax_{i\in\mathcal I}\widetilde S_i.
\]
Equivalently, one may draw
\(\widetilde Z=(\widetilde Z_i)_{i\in\mathcal I}\sim\Pi_t\), independently of
the actual residual variables, and use only the sampled scores to
choose the action.  We denote the resulting policy by \(\pi_{\mathrm{TS}}\).

For an admissible posterior \(\Pi\), we write
\[
        q_i(\Pi)=\Pp_\Pi(I^\star=i).
\]
By admissibility and the deterministic tie-breaking rule,
\(\sum_iq_i(\Pi)=1\).  Thompson sampling plays arm \(i\) next with probability
\(q_i(\Pi)\).  If the next action is fixed to be \(i\), the probability of a
mistake is \(1-q_i(\Pi)\).

For such a posterior \(\Pi\), a policy \(\pi\), and \(m\ge0\), we define
\[
        L_m^\pi(\Pi)
        =
        \E_\Pi^\pi\!\left[
        \sum_{t=1}^m \1_{\{A_t\ne I^\star\}}
        \right],
        \qquad
        L_m^\star(\Pi)=\inf_\pi L_m^\pi(\Pi),
        \qquad
        L_m^{\mathrm{TS}}(\Pi)=L_m^{\pi_{\mathrm{TS}}}(\Pi).
\]
The infimum in \(L_m^\star(\Pi)\) ranges over all \(m\)-round policies.
At the initial prior, we abbreviate
\[
        L_T^\pi=L_T^\pi(\Pi_0),
        \qquad
        L_T^\star=L_T^\star(\Pi_0),
        \qquad
        L_T^{\mathrm{TS}}=L_T^{\mathrm{TS}}(\Pi_0).
\]

\subsection{Main theorem}

Our main result is that Thompson sampling makes at most twice the average number of mistakes as any other policy.
\cite{guha-munagala-2014} observe that the factor of two is sharp even for two-arm stochastic bandits.

\begin{theorem}\label{thm:main}
Let \(\mathcal I\) be finite or countable, and suppose the initial prior
\(\Pi_0\) is admissible.  Then for every finite horizon \(T\),
\[
        L_T^{\mathrm{TS}}\le 2L_T^\star.
\]
\end{theorem}

\subsection{Examples}

The standard Bayesian stochastic bandit is recovered by setting \(S_i=\mu_i\)
and letting \(O_{i,k}\) be the \(k\)-th reward from arm \(i\)
\cite{guha-munagala-2014}.  Several other examples are described below.

\paragraph{Variance and covariance.}
Suppose arm \(i\) has an unknown covariance matrix \(\Sigma_i\) and returns independent
samples from (say) \(\mathcal N(0,\Sigma_i)\).  Taking \(S_i=\lVert\Sigma_i\rVert\) or
\(S_i=-\lVert\Sigma_i\rVert\) for any fixed matrix norm, the ``best'' arm is the most or least volatile one, in a precise sense depending on the norm.

\paragraph{Unknown transition times.}
Letting \(\tau_i\) be a random nonnegative integer for each $i$, suppose arm \(i\) returns zero on its
first \(\tau_i\) plays and one thereafter.  Taking \(S_i=-\tau_i\) makes the arm
with the earliest transition best, while \(S_i=\tau_i\) makes the arm with the
latest transition best.

\paragraph{Sampling without replacement.}
Suppose arm \(i\) contains \(m\) hidden bits \(X_{i,1},\ldots,X_{i,m}\), and each play
reveals one uniformly chosen unrevealed bit. After all \(m\) bits have been seen,
further plays reveal nothing new.  With
\(S_i=m^{-1}\sum_{k=1}^m X_{i,k}\), the best arm is the one with the largest
fraction of ones.

\paragraph{Rested Markov rewards.}
Suppose arm \(i\)'s observations follow a finite-state irreducible Markov reward process
that advances only when the arm is played. We can then take \(S_i\) to be its stationary mean
reward; such models are studied in \cite{anantharam-varaiya-walrand-1987-markov,tekin-liu-2012}.
By contrast, \cite{guha-munagala-shi-2010} studies restless models, in which an arm continues
to evolve while unplayed.

\subsection{Random horizons and round weights}

By a simple averaging argument, the factor-two guarantee also extends to a random horizon drawn independently of
the unknown arms and not revealed to the learner in advance.  Equivalently,
it extends to mistake loss with any nonincreasing sequence of deterministic
round weights.  Geometric discounting is the special case
\(w_t=\gamma^{t-1}\), with \(0<\gamma<1\).

\begin{corollary}\label{cor:random-horizon}
Let \(N\) be an \(\mathbb N\cup\{0,\infty\}\)-valued random horizon, independent
of the unknown arms and policy randomization.  Policies may use the distribution of
\(N\) and the current round, but not the realized value of \(N\).  Then
\begin{equation}
        \E^{\mathrm{TS}}\!\left[
        \sum_{t=1}^N\1_{\{A_t\ne I^\star\}}
        \right]
        \le
        2\inf_\pi
        \E^\pi\!\left[
        \sum_{t=1}^N\1_{\{A_t\ne I^\star\}}
        \right].
        \label{eq:random-horizon-bound}
\end{equation}
Equivalently, for any deterministic sequence
\(w_1\ge w_2\ge\cdots\ge0\),
\begin{equation}
        \E^{\mathrm{TS}}\!\left[
        \sum_{t\ge1} w_t\1_{\{A_t\ne I^\star\}}
        \right]
        \le
        2\inf_\pi
        \E^\pi\!\left[
        \sum_{t\ge1} w_t\1_{\{A_t\ne I^\star\}}
        \right].
        \label{eq:weighted-bound}
\end{equation}
\end{corollary}

\begin{proof}
We fix a policy \(\pi\) for the random-horizon problem and an integer
\(m\ge1\).  For each \(k\in\{0,\ldots,m\}\), Theorem~\ref{thm:main} and the
definition of \(L_k^\star\), applied to the restriction of \(\pi\) to its first
\(k\) decisions, give
\begin{equation}
        \E^{\mathrm{TS}}\!\left[
        \sum_{t=1}^k\1_{\{A_t\ne I^\star\}}
        \right]
        \le
        2\E^\pi\!\left[
        \sum_{t=1}^k\1_{\{A_t\ne I^\star\}}
        \right].
        \label{eq:fixed-truncation-bound}
\end{equation}
We average \eqref{eq:fixed-truncation-bound} over \(k=N\wedge m\), and then let
\(m\to\infty\).  Independence of \(N\) and monotone convergence yield
\eqref{eq:random-horizon-bound} after minimizing over \(\pi\).
Finally, if \(w_1>0\), \eqref{eq:weighted-bound} follows by applying
\eqref{eq:random-horizon-bound} to an independent hidden horizon \(N\) with law
\(\Pp(N\ge t)=w_t/w_1\).  The case \(w_1=0\) is immediate.
\end{proof}

\section{Proof of Theorem~\ref{thm:main}}\label{sec:proof-main}

We prove Theorem~\ref{thm:main} by induction on the horizon \(n\), using the
one-step estimate \eqref{eq:key-estimate} in
Proposition~\ref{prop:one-step}.  The induction is carried out uniformly over
every admissible posterior \(\Pi\):
\[
        L_n^{\mathrm{TS}}(\Pi)\stackrel{?}{\le} 2L_n^\star(\Pi).
\]
To prove \eqref{eq:key-estimate}, we couple the \(n\)-round problems following
free observations from different arms and bound the differences between their
\(L^\star\)-values.  We average these bounds over the posterior by swapping the
unobserved arms of two independent draws.

\subsection{First-step recursions and the one-step bound}

We fix \(n\ge0\) and an admissible posterior \(\Pi\).
Supposing that arm \(i\)'s next observation is revealed and consumed before the next \(n\) rounds begin, we write
\((\Upd_iL_n^\star)(\Pi)\) for the optimal expected number of mistakes over the
next \(n\) rounds, when the policy may use this free observation.  We define
\((\Upd_i L_n^{\mathrm{TS}})(\Pi)\) analogously for the expected number of mistakes made by Thompson sampling, after seeing this free observation.
(Thus in defining $\Upd_i$, the free observation counts as neither a round nor a mistake.)

With \(n+1\) rounds remaining, fixing the first action to be \(a\) gives
first-round mistake probability \(1-q_a(\Pi)\).
After \(O_{a,1}\) is revealed and consumed, the optimal expected number of
mistakes over the remaining \(n\) rounds is \((\Upd_a L_n^\star)(\Pi)\).
Hence
\begin{equation}
        L_{n+1}^\star(\Pi)
        =
        \inf_{a\in\mathcal I}
        \{1-q_a(\Pi)+(\Upd_aL_n^\star)(\Pi)\}.
        \label{eq:optimal-risk-recursion}
\end{equation}

Under Thompson sampling, the first action has distribution
\((q_i(\Pi))_{i\in\mathcal I}\), and the first-round mistake
probability is easily seen to be
\[
        \Pp_\Pi(\widetilde I^\star\ne I^\star)
        =
        1-\sum_{i\in\mathcal I} q_i(\Pi)^2
\]
where \(\widetilde I^\star\) is the best-arm index of an independent draw from \(\Pi\).
Averaging over the sampled index then gives
\begin{equation}
        L_{n+1}^{\mathrm{TS}}(\Pi)
        =
        1-\sum_{i\in\mathcal I} q_i(\Pi)^2
        +
        \sum_{i\in\mathcal I}
        q_i(\Pi)(\Upd_i L_n^{\mathrm{TS}})(\Pi).
        \label{eq:TS-risk-recursion}
\end{equation}
Combining \eqref{eq:optimal-risk-recursion} and
\eqref{eq:TS-risk-recursion} reduces Theorem~\ref{thm:main} to the following
one-step lower bound on \(L_{n+1}^\star(\Pi)\):

\begin{proposition}\label{prop:one-step}
For every admissible posterior \(\Pi\) and every \(n\ge0\),
\begin{equation}
        L_{n+1}^\star(\Pi)
        \ge
        \frac{1-\sum_{i\in\mathcal I}q_i(\Pi)^2}{2}
        +
        \sum_{i\in\mathcal I}q_i(\Pi)(\Upd_i L_n^\star)(\Pi).
        \label{eq:key-estimate}
\end{equation}
\end{proposition}

\begin{proof}[Proof of Theorem~\ref{thm:main} assuming Proposition~\ref{prop:one-step}]
We define
\[
        \Delta_n(\Pi)=2L_n^\star(\Pi)-L_n^{\mathrm{TS}}(\Pi).
\]
We use the same update notation
\[
        (\Upd_i\Delta_n)(\Pi)
        =2(\Upd_iL_n^\star)(\Pi)-(\Upd_iL_n^{\mathrm{TS}})(\Pi)
\]
and prove by induction on \(n\) that
\(\Delta_n(\Pi)\geq 0\) for every admissible
posterior \(\Pi\).  Since
\(\Delta_0(\Pi)=0\), the base case holds.  For the inductive step, we assume that
\(\Delta_n(\Pi')\ge0\) for every admissible posterior \(\Pi'\), and we fix an
admissible \(\Pi\).  If \(O_{i,1}=o\), let \(\Pi^{i,o}\) be the posterior after
this observation is consumed.  Conditioning on the free observation gives
\[
        (\Upd_iL_n^{\mathrm{TS}})(\Pi)
        =\E_\Pi L_n^{\mathrm{TS}}(\Pi^{i,O_{i,1}}),
        \qquad
        (\Upd_iL_n^\star)(\Pi)
        \ge\E_\Pi L_n^\star(\Pi^{i,O_{i,1}}).
\]
Indeed, after conditioning on \(O_{i,1}=o\), any policy using the free
observation is an \(n\)-round policy under \(\Pi^{i,o}\), and hence has expected
loss at least \(L_n^\star(\Pi^{i,o})\).  The updated posterior is admissible
almost surely, so the induction hypothesis gives
\begin{equation}
\begin{aligned}
        (\Upd_i\Delta_n)(\Pi)
        &=2(\Upd_iL_n^\star)(\Pi)
          -(\Upd_iL_n^{\mathrm{TS}})(\Pi)\\
        &\ge
        \E_\Pi\!\left[\Delta_n(\Pi^{i,O_{i,1}})\right]
        \ge0.
\end{aligned}
        \label{eq:updated-delta-nonnegative}
\end{equation}
Subtracting \eqref{eq:TS-risk-recursion} from twice
\eqref{eq:key-estimate} cancels the first-round terms and gives
\begin{equation}
\begin{aligned}
        \Delta_{n+1}(\Pi)
        &=
        2L_{n+1}^\star(\Pi)-L_{n+1}^{\mathrm{TS}}(\Pi) \\
        &\ge
        \sum_{i\in\mathcal I}
        q_i(\Pi)
        (\Upd_i\Delta_n)(\Pi).
\end{aligned}
        \label{eq:delta-recursion}
\end{equation}
Every term on the right-hand side of \eqref{eq:delta-recursion} is
nonnegative by \eqref{eq:updated-delta-nonnegative}.  This completes the
inductive step and hence the proof.
\end{proof}

\subsection{Proof of Proposition~\ref{prop:one-step}}\label{sec:key-proof}

It remains to prove Proposition~\ref{prop:one-step}.  We compare the problems obtained
from free observations of different single arms.  We fix \(n\ge0\) and an
admissible posterior \(\Pi\), and define
\[
        q_i=q_i(\Pi),
        \qquad
        H_i=(\Upd_iL_n^\star)(\Pi),
        \qquad
        C(q)=\frac{1-\sum_{i\in\mathcal I}q_i^2}{2}.
\]
Here \(H_i\) is the optimal expected number of mistakes over the next
\(n\) rounds when \(O_{i,1}\) is revealed and consumed before the first decision,
and \(C(q)\) is half the probability that two independent draws from \(\Pi\) have
different best-arm indices.  With this notation, \eqref{eq:key-estimate} becomes
\begin{equation}
        C(q)+\sum_{i\in\mathcal I} q_iH_i
        \stackrel{?}{\le} L_{n+1}^\star(\Pi).
        \label{eq:H-bound}
\end{equation}
Recalling \eqref{eq:optimal-risk-recursion}, it suffices to show that for every \(a\in\mathcal I\):
\begin{equation}
        C(q)+\sum_{i\in\mathcal I} q_i H_i
        \stackrel{?}{\le}
        1-q_a+H_a.
        \label{eq:key-pull}
\end{equation}
Here \(1-q_a+H_a\) is the optimal expected number of mistakes when the first
action is constrained to be \(a\).
When \(n=0\), every \(H_i\) is zero, while
\[
        1-q_a-C(q)
        =\frac12(1-q_a)^2
        +\frac12\sum_{i\ne a}q_i^2
        \ge0.
\]
Thus \eqref{eq:key-pull} holds for every \(a\) when \(n=0\), proving
Proposition~\ref{prop:one-step} in this case.  We now assume \(n\ge1\), fix an
arm \(a\), and prove \eqref{eq:key-pull}.

\paragraph{Coupling after different free observations.}

We fix any \(n\)-round policy \(F\) that receives \(O_{a,1}\) for free and
construct, for each \(i\), a policy
\(F^{(i)}\) that receives \(O_{i,1}\) for free.  We may assume that \(F\) is
deterministic since for a randomized \(F\), one can fix its private seed throughout the construction and average over the seed afterward.

We couple \(F\) and the policies \(F^{(i)}\) using one draw
\(X=(X_j)_{j\in\mathcal I}\sim\Pi\), where
\(X_j=(S_j,O_{j,1},O_{j,2},\ldots)\), and write \(a_1,\ldots,a_n\) for the
actions of \(F\) on \(X\).  We set \(F^{(a)}=F\).  For \(i\ne a\), \(F^{(i)}\)
stores its free observation \(O_{i,1}\), plays arm \(a\), and uses the resulting
\(O_{a,1}\) to initialize a simulation of \(F\).  Then \(F^{(i)}\) follows the
simulation until \(F^{(i)}\) has made \(n\) plays.  If the simulation first selects
arm \(i\) before \(F^{(i)}\) stops, \(F^{(i)}\) adds the stored observation to
the simulated history without making a play.

The initial play of \(a\) and the stored observation from \(i\) align the
arm-local observation counts, so \(F^{(i)}\) is a valid policy.  Set \(a_0=a\)
and \(\mathcal R=\{a_0,a_1,\ldots,a_n\}\).  It is easy to check that the \(n\)
plays of \(F^{(i)}\) are obtained from \(a_0,a_1,\ldots,a_n\) by deleting one
occurrence of the arm
\begin{equation}
        J_i=
        \begin{cases}
        i, & i\in\mathcal R,\\
        a_n, & i\notin\mathcal R.
        \end{cases}
        \label{eq:J-def}
\end{equation}
Since \(a_n\in\mathcal R\), both cases in \eqref{eq:J-def} give
\begin{equation}
        J_i\in\mathcal R
        \qquad\text{for every }i\in\mathcal I.
        \label{eq:J-in-R}
\end{equation}
If \(i\ne a\) and the stored observation is used, the deleted arm is \(i\);
otherwise \(F^{(i)}\) stops before \(a_n\), so the deleted arm is \(a_n\).
When \(i=a_n\), both cases give \(J_i=a_n\).  For \(i=a\), deleting the
inserted action \(a_0=a\) gives \(J_a=a\).

We denote the best-arm index of \(X\) by \(I_X^\star\), and write \(\ell(F)\)
and \(\ell(F^{(i)})\) for their respective mistake counts under this coupling;
the free observation before the \(n\) counted plays incurs no mistake.
Inserting a play of \(a\) and omitting a play of \(J_i\) changes the count by exactly
\begin{equation}
        \ell(F^{(i)})-\ell(F)
        =
        \1_{\{I_X^\star\ne a\}}-\1_{\{I_X^\star\ne J_i\}}.
        \label{eq:loss-diff}
\end{equation}
By the definition of \(H_i\), we have \(H_i\le\E\ell(F^{(i)})\).  Taking expectations
in \eqref{eq:loss-diff} gives
\begin{equation}
        H_i-\E\ell(F)
        \le
        \E\!\left[
        \1_{\{I_X^\star\ne a\}}-\1_{\{I_X^\star\ne J_i\}}
        \right].
        \label{eq:Hi-F-comparison}
\end{equation}

So far, our analysis of $F^{(i)}$ has not used the fact that we are studying Thompson sampling, but this will now become relevant.
Draw an independent posterior sample
\(Y=(Y_j)_{j\in\mathcal I}\sim\Pi\) and denote its best-arm index by
\(I_Y^\star\), so that \(\Pp(I_Y^\star=i)=q_i\).  
Writing \(J_i=J_i(X)\) to make its dependence on \(X\) explicit, the \(q_i\)-weighted average of
\eqref{eq:Hi-F-comparison} gives
\begin{equation}
        \sum_{i\in\mathcal I}q_iH_i-\E\ell(F)
        \le
        1-q_a-
        \Pp\!\left(J_{I_Y^\star}(X)\ne I_X^\star\right).
        \label{eq:averaged-comparison}
\end{equation}

\paragraph{Swapping unobserved arms.}

The next lemma shows the remaining estimate.  Once \eqref{eq:crossing} is
proved, substituting it into \eqref{eq:averaged-comparison} and taking the
infimum over \(F\) gives \eqref{eq:key-pull}, and hence
Theorem~\ref{thm:main}.  We prove this lemma by swapping the unobserved
arms between \(X\) and \(Y\).

\begin{lemma}
For independent \(X,Y\sim\Pi\), with \(J_i(X)\) defined by
\eqref{eq:J-def},
\begin{equation}
        \Pp(J_{I_Y^\star}(X)\ne I_X^\star)
        \ge
        \frac12\Pp(I_Y^\star\ne I_X^\star)
        =C(q).
        \label{eq:crossing}
\end{equation}
\end{lemma}

\begin{proof}
We write
\begin{equation}
        \mathcal A=\{J_{I_Y^\star}(X)\ne I_X^\star\}.
        \label{eq:A-def}
\end{equation}
Thus \eqref{eq:crossing} asks us to prove
\[\Pp(\mathcal A)\stackrel{?}{\ge}
\frac12\Pp(I_X^\star\ne I_Y^\star).
\]

To compare $I_X^\star$ and $I_Y^\star$, including when their scores tie, we use
the smallest-index convention in \eqref{eq:best-arm}.
For a posterior draw \(W\), let \(S_j^W\) denote the score of arm \(j\) in
that draw.  We order index--score pairs by declaring
\[
        (j,s)\succ(k,r)
\]
if \(s>r\), or if \(s=r\) and \(j<k\).  The best arm in \(W\) is precisely the
index of the largest pair \((j,S_j^W)\) under \(\succ\).

This order splits the event \(I_X^\star\ne I_Y^\star\) according to whether
the best-arm pair from \(X\) or the best-arm pair from \(Y\) is larger.  The
two cases have equal probability by exchangeability, which explains the
factor \(1/2\) in \eqref{eq:crossing}.  On the \(Y\)-larger outcomes outside
\(\mathcal A\), the swap below places \(Y_{I_Y^\star}\) in \(X'\), where
\(I_Y^\star\) becomes the best arm.

We define
\begin{equation}
        E_{Y>X}
        =
        \{I_X^\star\ne I_Y^\star
        \text{ and }(I_Y^\star,S_{I_Y^\star}^Y)
        \succ(I_X^\star,S_{I_X^\star}^X)\}.
        \label{eq:E-def}
\end{equation}
Since \((X,Y)\) and \((Y,X)\) have the same law and the two orderings partition
the event \(I_X^\star\ne I_Y^\star\),
\begin{equation}
        \Pp(E_{Y>X})
        =
        \frac12\Pp(I_X^\star\ne I_Y^\star)
        \label{eq:EYhalf}
\end{equation}
Thus it suffices to prove
\begin{equation}
        \Pp(\mathcal A)\stackrel{?}{\ge}\Pp(E_{Y>X}).
        \label{eq:A-dominates-E}
\end{equation}
We split $E_{Y>X}$ as
\[
        E_{Y>X}
        =
        (E_{Y>X}\cap\mathcal A)\,\dot\cup\,B,
        \qquad
        B:=E_{Y>X}\cap\mathcal A^c.
\]
The set \(E_{Y>X}\cap\mathcal A\) already lies in \(\mathcal A\).  We will map
\(B\) into \(\mathcal A\setminus E_{Y>X}\) without changing its probability.
The image of \(B\) and \(E_{Y>X}\cap\mathcal A\) will then be disjoint subsets
of \(\mathcal A\), proving \(\Pp(\mathcal A)\ge\Pp(E_{Y>X})\).

Recall that \(\mathcal R=\{a_0,\ldots,a_n\}\), where \(a_0=a\) supplies
\(F\)'s free observation and \(a_1,\ldots,a_n\) are the actions selected by
\(F\) on \(X\).
We define \((X',Y')=\Phi(X,Y)\) by keeping \((X_j,Y_j)\) fixed for
\(j\in\mathcal R\) and swapping \(X_j\) with \(Y_j\) for every
\(j\notin\mathcal R\):
\[
\begin{array}{c|cc}
& j\in\mathcal R & j\notin\mathcal R \\ \hline
X'_j & X_j & Y_j\\
Y'_j & Y_j & X_j .
\end{array}
\]
The set \(\mathcal R\) depends on \(X\).  Let \(\mathcal T\) be the finite
transcript of \(F\), consisting of its free observation and the \(n\)
action--observation pairs.  Since \(F\) is deterministic, conditioning on
\(\mathcal T\) fixes \(\mathcal R\), and no part of \(X_j\) has been observed for
\(j\notin\mathcal R\).  The product structure of \(\Pi\) and the independence
of \(Y\) therefore make the pairs \((X_j,Y_j)\), \(j\notin\mathcal R\),
conditionally independent and exchangeable.  Thus \(\Phi\) preserves the
conditional law, and hence the joint law, of \((X,Y)\).

Because \(\Phi\) leaves \(X_j\) unchanged for \(j\in\mathcal R\), the
transcript under \(X'\) is again \(\mathcal T\).  A second application of
\(\Phi\) therefore swaps the same coordinates and restores \((X,Y)\), so
\(\Phi^2=\mathrm{id}\).

Set
\[
        B'=\Phi(B)=\{\Phi(x,y):(x,y)\in B\}.
\]
Since \(\Phi\) is a probability-preserving involution,
\begin{equation}
        \Pp(B')=\Pp(B).
        \label{eq:swap-preserves-B}
\end{equation}
To prove \eqref{eq:A-dominates-E}, it remains to show
\begin{equation}
        B'\stackrel{?}{\subseteq}\mathcal A\setminus E_{Y>X}.
        \label{eq:swap-image-inclusion}
\end{equation}
Fix \((X,Y)\in B\), and write \((X',Y')=\Phi(X,Y)\).  We will verify
separately that \((X',Y')\in\mathcal A\) and that
\((X',Y')\notin E_{Y>X}\), which are the two requirements in
\eqref{eq:swap-image-inclusion}.

Because \(B=E_{Y>X}\cap\mathcal A^c\), the definition \eqref{eq:E-def} gives
\begin{equation}
        I_Y^\star\ne I_X^\star,
        \qquad
        (I_Y^\star,S_{I_Y^\star}^Y)
        \succ
        (I_X^\star,S_{I_X^\star}^X),
        \label{eq:B-order}
\end{equation}
while \eqref{eq:A-def} gives
\begin{equation}
        J_{I_Y^\star}(X)=I_X^\star.
        \label{eq:B-not-A}
\end{equation}
If \(I_Y^\star\in\mathcal R\), then \eqref{eq:J-def} would give
\(J_{I_Y^\star}(X)=I_Y^\star\).  Comparing this with
\eqref{eq:B-not-A} would force \(I_Y^\star=I_X^\star\), contrary to
\eqref{eq:B-order}.  Hence \(I_Y^\star\notin\mathcal R\), so the second case
of \eqref{eq:J-def} gives \(J_{I_Y^\star}(X)=a_n\).  Comparing again with
\eqref{eq:B-not-A}, we obtain
\begin{equation}
        I_Y^\star\notin\mathcal R,
        \qquad
        I_X^\star=a_n\in\mathcal R.
        \label{eq:best-arm-memberships}
\end{equation}

We next show that after the swap, the best arm lies on the \(X'\) side.  The
best arm under \(Y\) is better than every other arm under \(Y\).  By
\eqref{eq:B-order}, it is also better than the best arm under \(X\), and
hence than every arm under \(X\).  It is therefore best among all the arm
data appearing in \(X\) and \(Y\).  The swap only redistributes these data
between \(X'\) and \(Y'\).  Since \(I_Y^\star\notin\mathcal R\), the data for
this best arm move to \(X'\).  Thus it becomes the best arm under \(X'\) and
remains better than every arm under \(Y'\).  In particular,
\begin{equation}
        I_{X'}^\star=I_Y^\star\notin\mathcal R.
        \label{eq:best-under-Xprime}
\end{equation}
The definition \eqref{eq:E-def} therefore immediately gives
\((X',Y')\notin E_{Y>X}\).

Under \(X'\), \(F\) again selects \(a_1,\ldots,a_n\), so the set used to
define \(J_i(X')\) is again \(\mathcal R\).  Applying
\eqref{eq:J-in-R} with \(i=I_{Y'}^\star\) gives
\[
        J_{I_{Y'}^\star}(X')\in\mathcal R.
\]
On the other hand, \eqref{eq:best-under-Xprime} gives
\(I_{X'}^\star\notin\mathcal R\).  Therefore
\(J_{I_{Y'}^\star}(X')\ne I_{X'}^\star\), which by \eqref{eq:A-def} means
precisely that \((X',Y')\in\mathcal A\).
Since \((X,Y)\in B\) was arbitrary, we have proved
\eqref{eq:swap-image-inclusion}.

It follows that \(E_{Y>X}\cap\mathcal A\) and \(B'\) are disjoint subsets of
\(\mathcal A\), and therefore
\[
        \Pp(\mathcal A)
        \ge \Pp(E_{Y>X}\cap\mathcal A)+\Pp(B').
\]
Recall from \eqref{eq:swap-preserves-B} that \(\Pp(B')=\Pp(B)\) since $\Phi$ is 
probability preserving; and by the definition of
\(B\), the sets \(E_{Y>X}\cap\mathcal A\) and \(B\) partition \(E_{Y>X}\).
Therefore
\[
        \Pp(E_{Y>X}\cap\mathcal A)+\Pp(B')
        =
        \Pp(E_{Y>X}\cap\mathcal A)+\Pp(B)
        =
        \Pp(E_{Y>X}).
\]
Combining the last two displays and then applying \eqref{eq:EYhalf} yields, as desired,
\[
        \Pp(\mathcal A)\ge\Pp(E_{Y>X})
        =\frac12\Pp(I_X^\star\ne I_Y^\star)
        =C(q).\qedhere
\]
\end{proof}

\begingroup
\fontsize{9.3}{9.5}\selectfont
\raggedright
\bibliographystyle{alpha}
\bibliography{bib}
\endgroup

\end{document}